\documentclass[letterpaper, 10 pt, conference]{ieeeconf}  

\IEEEoverridecommandlockouts                              

\usepackage{cite}
\usepackage{amsmath,amssymb,amsfonts}
\usepackage{algorithmic}
\usepackage{graphicx}
\usepackage{textcomp}
\usepackage{xcolor}
\usepackage{blindtext}
\usepackage{graphicx}
\usepackage{capt-of}

\usepackage{bm}
\usepackage{lipsum}
\usepackage{mwe}
\usepackage{babel}
\usepackage{cuted}
\usepackage{url}
\usepackage{hyperref}
\hypersetup{
    colorlinks=true,
    linkcolor=black,
    filecolor=magenta,    
    citecolor=black,
    urlcolor=blue,
    }
\usepackage{balance}

\newtheorem{definition}{Definition}
\newtheorem{lemma}{Lemma}
\newtheorem{problem}{Problem}
\newtheorem{theorem}{Theorem}

\newtheorem{remark}{Remark}

\newtheorem{assumption}{Assumption}

\graphicspath{{images/}}
\allowdisplaybreaks
\def\b{\boldsymbol}
\def\be{\boldsymbol e}

\def\bu{\boldsymbol u}
\def\bp{\boldsymbol p}
\def\bq{\boldsymbol q}

\def\bp{\boldsymbol p}

\def\R{\mathbf R}
\def\Z{\mathbf Z}
\def\Log{\mathsf{Log}}

\newcommand{\norm}[1]{\left\|#1\right\|}

\title{\LARGE \bf
Autonomous 3D Moving Target Encirclement and Interception with Range Measurement
}

\author{Fen Liu,~Shenghai Yuan,~\textit{Member,~IEEE},~Thien-Minh Nguyen,~\textit{Member,~IEEE},~Rong Su,~\textit{Senior Member,~IEEE}
 \thanks{The work is supported by National Research Foundation of Singapore under its Medium-Sized Center for Advanced Robotics Technology Innovation and by Naval Group Far East Pte Ltd via an RCA with NTU.}
\thanks{F. Liu, S. Yuan, T. Nguyen and  R. Su are with the School of Electrical and Electronic Engineering, Nanyang Technological University, Singapore 639798, Singapore (e-mail: {fen.liu, shyuan, thienminh.nguyen, RSu}@ntu.edu.sg).}
}

\begin{document}


 \maketitle

\begin{strip}
\begin{minipage}{\textwidth}\centering
\vspace{-40pt}
\centering
\includegraphics[width=1\textwidth]{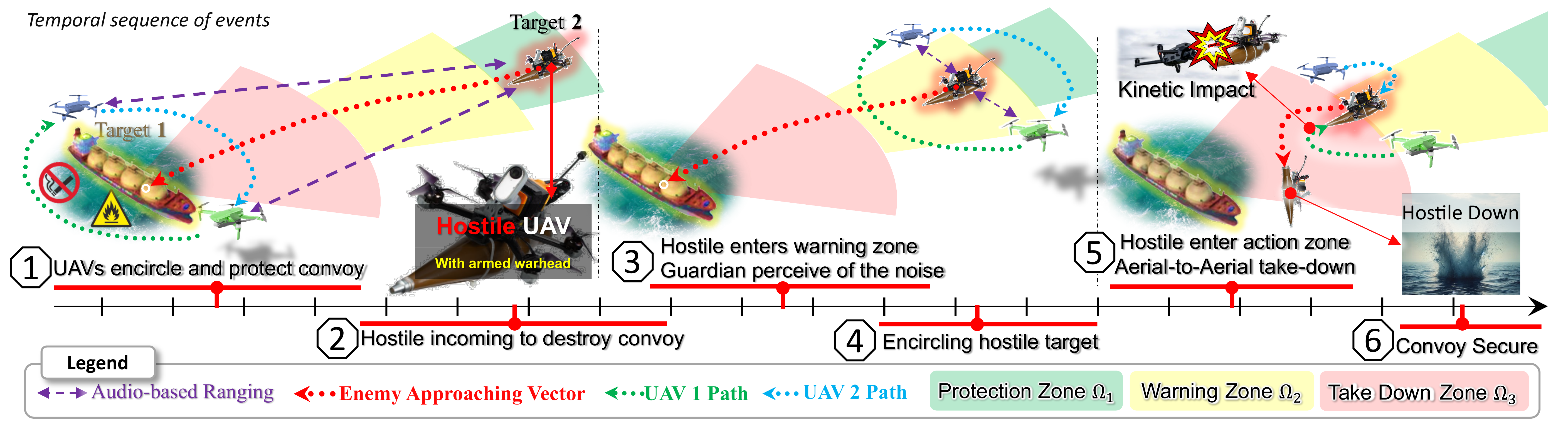}
\vspace{-15pt}
\captionof{figure}{Illustrations of the drone swarm overwatch concepts for convoy escort of Target 1.}
\label{overarching}
\vspace{-10pt}
\end{minipage}

\end{strip}

\thispagestyle{plain}
\pagestyle{plain}

\begin{abstract}
Commercial UAVs are an emerging security threat as they are capable of carrying hazardous payloads or disrupting air traffic. To counter UAVs, we introduce an autonomous 3D target encirclement and interception strategy. Unlike traditional ground-guided systems, this strategy employs autonomous drones to track and engage non-cooperative hostile UAVs, which is effective in non-line-of-sight conditions, GPS denial, and radar jamming, where conventional detection and neutralization from ground guidance fail. Using two noisy real-time distances measured by drones, guardian drones estimate the relative position from their own to the target using observation and velocity compensation methods, based on anti-synchronization (AS) and an X$-$Y circular motion combined with vertical jitter.
An encirclement control mechanism is proposed to enable UAVs to adaptively transition from encircling and protecting a target to encircling and monitoring a hostile target. Upon breaching a warning threshold, the UAVs may even employ a suicide attack to neutralize the hostile target.
We validate this strategy through real-world UAV experiments and simulated analysis in MATLAB, demonstrating its effectiveness in detecting, encircling, and intercepting hostile drones. More details: \url{https://youtu.be/5eHW56lPVto}.
\end{abstract}

\section{Introduction}
The rise of commercial drones with long endurance \cite{cao2023neptune} and advanced navigation capabilities \cite{cao2022direct,wu2021learn} has introduced security challenges, calling for effective countermeasures \cite{jacobsen2021security}. Traditional defensive systems rely on ground-based visual \cite{xia2024av-dtec} or radar tracking \cite{lei2025audio}, which are optimized for large objects. However, their limitations in detecting small UAVs make them less effective against low-altitude compact drones.

Newer methods \cite{yasmine2022survey}, such as radio frequency jamming, GPS spoofing, and net-catching \cite{vrba2024onboard}, attempt to address these threats but suffer from issues such as high power consumption, immobility, etc. To overcome these issues, the latest anti-drone concept \cite{lefebvre2016conceptual} proposes drone to drone combat system, offering an interesting way to neutralize threats beyond the horizon.
While manual drones have been used in conflicts like Ukraine for counter-drone operations, transitioning to autonomous aerial drones poses challenges. Target localization is constrained by payload limits, onboard sensors \cite{yuan2021survey}, and computational restrictions, making lightweight sensors the most practical solution.

Existing target localization algorithms often rely on cameras for bearing or relative angle measurement \cite{li2022three, chen2022triangular}, but this requires a complex visual processing pipeline. In contrast, distance measurement is generally simpler and more practical. However, ranging systems like UWB \cite{fang2020graph, wang2017ultra} or WiFi \cite{yang2024mm} depend on cooperative targets, limiting their versatility in real-world scenarios.
To address non-cooperative hostile targets, recent ranging solutions \cite{yang2023av, yuan2024MMAUD} employ microphones, enabling robust localization independent of lighting or radio frequency conditions. Despite these advantages, range-only localization algorithms face inherent limitations, often assuming the target is stationary, moving at a constant speed, or moving slowly \cite{shames2011circumnavigation, jiang2016simultaneous, dong2020target}. When targets exhibit unpredictable or high-speed movement, additional estimators or compensators become necessary for accurate motion tracking.

While localization is crucial, neutralization remains an even greater challenge in aerial anti-drone operations. Research in joint aerial anti-drone estimation and neutralization is still emerging, with limited studies available. Many existing neutralization techniques \cite{souli2023multi,9636065,9340835} rely on detecting radio control frequencies, but these approaches lack precision and are easily circumvented by advanced technologies such as 5G-datalink \cite{fan2024air}, optical communications \cite{singh2022comprehensive}, or tethered control systems \cite{cao2023neptune}.
Alternative approaches, such as deep reinforcement learning-based interception \cite{hu2022pursuit, li2023predator, wu2019tdpp}, require extensive training data and high onboard computing power, making them impractical and costly for real-world deployment. Additionally, most existing control techniques \cite{li2022three, nguyen2019distance, nguyen2019persistently, vrba2024onboard, simonsen2020application, walter2019uvdar} focus solely on target detection and following, lacking the capability for direct engagement. As a result, these methods are ineffective against agile or evasive drones, underscoring the need for more adaptive and robust anti-drone solutions.

This paper presents a novel approach to autonomous drone-based overwatch, interception, and neutralization of unauthorized or malicious drones, eliminating the need for complex surface vehicle guided systems. To the best of our knowledge, this work is the first to demonstrate autonomous drone-to-drone interception using minimal onboard sensors with range-only measurements.
The key contributions of this paper are summarized as follows:
\begin{enumerate}

  \item 
We propose a method for estimating the relative position of a non-cooperative target using noisy distance measurements from onboard sensors of two drones, without requiring prior knowledge of the target’s motion. Unlike previous approaches \cite{nguyen2018integrated, nguyen2019persistently}, this estimator can achieve persistently exciting (PE)-based estimation by leveraging the distribution of UAVs and self-control-compensated observations.
  
  \item 
  We propose a novel anti-drone controller (ADC) that utilizes the anti-synchronization (AS) encirclement strategy \cite{liu2023multiple,liu2023non} and an X$-$Y circular motion combined with vertical jitter. Unlike direct target guidance methods \cite{shah2014guidance,li2022three}, our controller compensates for hostile target observations while encircling the protected target, as well as adaptively encircling the hostile target. 
  
\end{enumerate}

\section{Problem formulation} \label{Problem}
Our mission focuses on providing overwatch to the protected target from the threats posed by the non-cooperative hostile target. To this end, we will deploy two guardian drones, designated as Drone 1 and Drone 2. A protected target can be any vehicle or boat, but a hostile target is strictly considered a drone with hostile intentions. 

We denote the position of drone $i$ as $\bp_i$. Hence, the kinematic model of the UAVs is defined as follows \cite{forster2015imu}:
\begin{align} \label{eq: dynamic model}
    R_i(k+1) &= \mathrm{Exp}\big( \bar{\bu}_i(k)t \big) R_i(k),
    \\
    \bp_i(k+1) &= \bp_i(k) + t R_i(k){\bu}_{i}(k),
\end{align}
where $t$ is the sampling time, and $R_i = {}^w_bR_i \in SO(3)$ denotes the rotation matrix that transforms the control input in the drone body frame $b$ to the world frame $w$. The vectors $\bu_i(k) \in \mathbb{R}^3$ and $\bar{\bu}_i = [\bar{\bu}_{i,\phi}, \bar{\bu}_{i,\theta}, \bar{\bu}_{i,\psi}]^\top \in \mathbb{R}^3$ denote the linear velocity and angular velocity of drone $i$ in the body frame, respectively. $\phi,\theta, \psi$ are attitude angles.

Based on \eqref{eq: dynamic model}, the relative position between drone 1 and drone 2, denoted as $\bq_{12}$, is obtained by taking the difference between $\bq_1$ and $\bq_2$ at any time instance $k$, which yields:
\begin{align}\label{eq:1-1}
{\bq}_{12}(k+1) &= {\bq}_{12}(k) \nonumber\\
                &+ t \Big[ R_1(k) \cdot \bu_{1}(k)
                 - R_2(k) \cdot \bu_{2}(k) \Big].
\end{align}


For notational convenience, the time index $(k)$ will be partially omitted in the following discussion. Specifically, $(k+1)$ will be represented as $^{(+)}$ and $(k-1)$ as $^{(-)}$.

We define the relative motion model for the target $j,j\in\{1,2\}$ and drone $i$ in the world frame as
\begin{equation}\label{eq1-5}
{\bq_i^j}^{(+)} = {\bq}_i^j + t(R_i\bu_{i}-\b{v}_j),\ i \in \{1,2\},
\end{equation}
where ${\bq}_i^j \in \R^3$ is the relative position from drone $i$ to the target $j$ at time $k$, and $\b{v}\in \R^3$ is the unknown velocity of target. Here, $\b{v}$ can follow an arbitrary distribution and includes some environmental noise and own nonlinear maneuvers, but is assumed to satisfy $\norm{\b{v}_j }\leq \check{v}_j$ with known $\check{v}_s \in \R^{+}$. 

We define the relative distances between the drones or targets are denoted as follows,
\begin{align}
    d_{12} \triangleq \norm{\b{q}_{12}},\ d_{i}^j \triangleq \norm{\b{q}_{i}^j}, d_{12} \triangleq \norm{\b{q}_{i}^2-\b{q}_{i}^1},
\end{align}
where the subscript is reserved for the index of the guardian and the superscript for the target.

In this work, we focus on dealing with a three-dimensional non-cooperative Target 2. Therefore, we assume that Target 1 will share its position with both drones.
For simplicity, when the protected target is a surface robot, we define Target 1 as a virtual point at a height $h_i$ above the real target, i.e. $\check{z}_1= h$ for Drones.

Denote $\hat{s}_2$ and $\hat{d}^{12}$ as the estimate of $s_2$ and ${d}^{12}$, which will be furnished via some estimation laws. We define three decision zones based on $\hat{d}^{12}$, namely $\Omega_1=\{\hat{d}^{12}|\hat{d}^{12}\geq z_3\}$, $\Omega_2=\{\hat{d}^{12}|z_3> \hat{d}^{12}\geq z_2\}$ and $\Omega_3=\{\hat{d}^{12}|z_2>\hat{d}^{12}\geq z_1\}$, where $z_1$, $z_2$ and $z_3$ are the radii of the \textit{target protection zone}, \textit{take down (or capture) zone} and \textit{warning zone}, respectively, as depicted in Fig. \ref{overarching}.

Moreover, we define an \textit{encirclement shape} vector $\bm{\zeta}(r,\nu,k) \in \mathrm{R}^3$ as
\begin{equation}
\begin{split}
\bm{\zeta}(r,\nu,k)
=
r(k)
\begin{bmatrix}
    \sin(\nu k\pi), &\cos(\nu k\pi), &\bm{g}(k)
\end{bmatrix}^\top, 
\end{split}
\end{equation}
where $r(k)$ is the encirclement radius, and $\bm{g}(k) \in \R$ is the vertical motion function. $\{\bm{\zeta}(k)\}$ satisfies the following PE assumption. 

\begin{assumption}\label{Persistently_exciting_shape}
The sequence $\{\bm{\zeta}(k)\}$ is PE, i.e. there exist $\hat{a}_\zeta, \check{a}_\zeta \in \R^{+} $ and $N \in \Z^+$ such that: \vspace{-2mm}
\begin{equation*}
\hat{a}_\zeta I
\leq\sum_{m=k}^{k+N-1}\bm{\zeta}(m) \bm{\zeta}(m)^\top\\
\leq \check{a}_\zeta I , \forall k \geq 0.
\end{equation*}
\end{assumption}

\begin{definition}
The guardian drones are said to encircle Target $j$ in an anti-synchronization (AS) manner when the preset encirclement shape $\bm{\zeta}(r,\nu,k)$ and the relative positions $\b{q}_{i}^{j}$, $i, j \in \Phi_2$ satisfy $\bm{\zeta}(r,\nu,k)^\top\b{q}_{1}^{j}=-\norm{\bm{\zeta}(r,\nu,k)}\norm{\b{q}^j_1}$ and $\bm{\zeta}(r,\nu,k)^\top\b{q}_{2}^{j}=\norm{\bm{\zeta}(r,\nu,k)}\norm{\b{q}^j_2}$.
\end{definition}

In other words, the relative positions of the guardian drones from the target are opposite to each other.


We can now make a statement for our problem of interest as follows:

\begin{problem}
Design a range-based estimation law for Target 2's position. Simultaneously, design the control laws $\bar{\b{u}}_i$ to ensure that the coordinate frames of all UAVs align with the world coordinate system, and design the control laws $\b{u}_i$ for the guardian drones to encircle Target 1 when $\hat{d}^{12} \in \Omega_1$, and \textit{encircle} or \textit{take down} Target 2 when $\hat{d}^{12}$ enters $\Omega_2$ and $\Omega_3$, respectively. The estimation and control objectives under the proposed laws can be expressed mathematically as:
\begin{subequations}
\begin{align}
R_1&=R_2=I,\label{eq1-11-0}\\
\underset{k\rightarrow \infty}{\mathrm{lim}}&\norm{{\bq}_i^2-\hat{\bq}_i^2}^2\leq\varrho_1,\label{eq1-11-1}\\
\underset{k\rightarrow \infty}{\mathrm{lim}}&\norm{\b{q}_{1}^{j}+\bm{\zeta}(r,\nu,k)}^2\leq\varepsilon_{2,j},\label{eq1-11-3}\\
\underset{k\rightarrow \infty}{\mathrm{lim}}&\norm{\b{q}_{2}^{j}-\bm{\zeta}(r,\nu,k)}^2\leq\varepsilon_{2,j},\label{eq1-11-4}
\end{align}
\end{subequations}
where $\varrho_1$ is the estimation error bound for Target 2. $\varepsilon_{1,j}$ and $\varepsilon_{2,j}$ represent the AS-based encirclement tracking error bound and the encirclement error bound for Target $j$, respectively. The difference between encirclement and capture is based on the radius $r$. When $r \geq \bar{r}$, we say the objective is to encircle and capture otherwise. Note that for surface targets, we only ensure that the two drones achieve encirclement in the XY plane.

\end{problem}

\section{Estimator and controller design}
\subsection{Measurement}
In this work, the distance $d_{i}^2$ can be measured by the on-board sensors of Drone $i$, e.g., by equirectangular perception module, wireless signal strength indicators, or target sound noise strength captured by a microphone \cite{yang2023av}. Some of these ideas have been validated as proof of concept \cite{cabrera2020detection}. 
Target 1 shares its position information with the two drones. Moreover, communication exists between the two UAVs, allowing them to share visual data as well as attitude angles obtained from the fusion of IMU and compass measurements. This ensures the alignment of coordinate frames between UAVs and enables the acquisition of the relative position $\bq_{12}$ (e.g. visual SLAM or UWB array).

 \begin{figure}
\centering
  \includegraphics[width=1\linewidth]{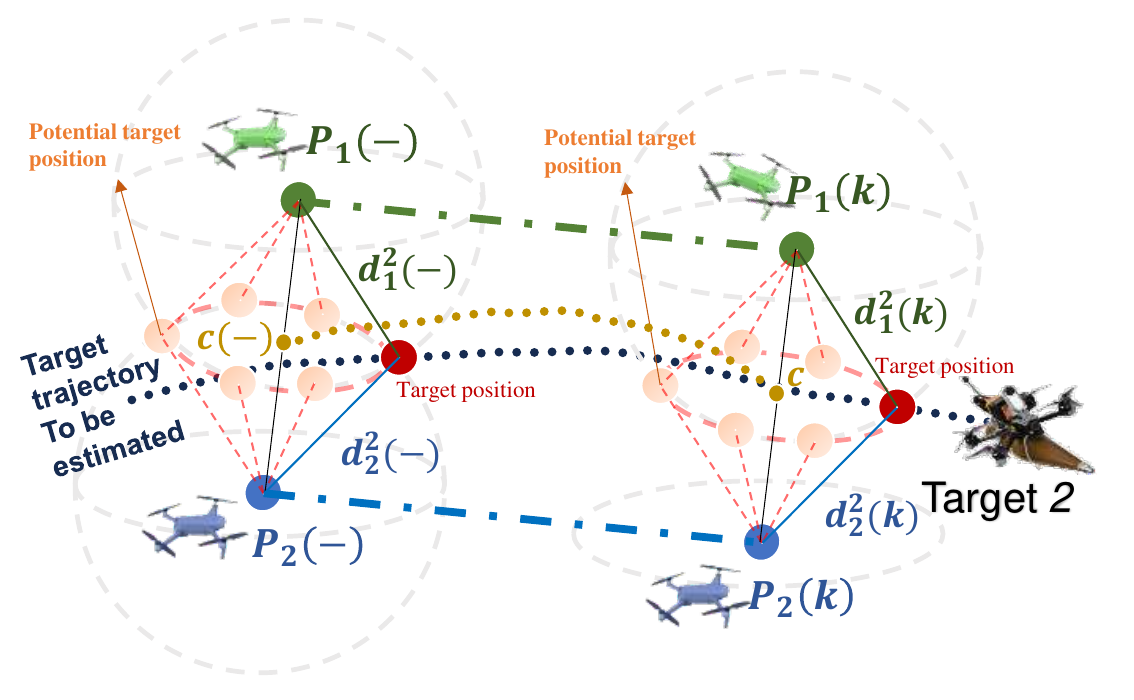}
 \caption{Range-based target motion analysis.}
  \label{target_motion}
  \vspace{-10pt}
\end{figure}

Given that $d_{i}^2$ and $\bp_i$ can be measured, the target position may exist at any point on the sphere with center $\bp_i$ and radius $d_{i}^2$, i.e. $(x-\bar{x}_i)^2+(y-\bar{y}_i)^2+(z-\bar{z}_i)^2=(d_{i}^2)^2$. Considering the presence of two drones, the position of the target can be further constrained to any point on the circle where the spheres intersect, as illustrated in Fig. \ref{target_motion}. Define the radius and the center of the intersection circle as $\varsigma$ and $\b{c}$, respectively. Considering $\norm{\bp_1-\b{c}}^2+\varsigma^2=(d_{1}^2)^2$ and $\norm{\bp_2-\b{c}}^2+\varsigma^2=(d_{2}^2)^2$, the radius $\varsigma$ and the center $\b{c}$ of the intersection circle can be calculated as: 
\begin{equation}
\varsigma=\frac{\sqrt{4d_{12}^2(d_{1}^2)^2-((d_{1}^2)^2-(d_{2}^2)^2+d_{12}^2)^2}}{2d_{12}},
\end{equation}
\begin{equation}\label{eq1-4}
\b{c}=\bp_1+\frac{\sqrt{(d_{1}^2)^2-\varsigma^2}}{d_{12}}\bp_{21}.
\end{equation}

Furthermore, based on the measurement distance $d_{i}^2$ and the position $\bp_i$, one output variable $\varpi_{i, p,2}$ related to the position of Target $2$ can be obtained for $(d_{i}^2)^2=\bp_i^\top\bp_i-2\bp_i^\top{\bq}_i^2+({\bq}_i^2)^\top{\bq}_i^2$, i.e.
\begin{equation}\label{eq1-3}
\begin{split}
\varpi_{i, p,2}\triangleq&\b{q}_{12}^\top{\bq}_i^2\\
=&\frac{1}{2}((d_{1}^2)^2-(d_{2}^2)^2-\bp_{12}^\top\bp_{12}).
\end{split}
\end{equation}

The above observation model will be used to design the Target Postion Estimator (TPE) in the subsequent part.



\subsection{Position estimation} \label{sec: tpe}
In this work, the center displacement of the intersection circle, i.e. $\b{c}-\b{c}^{(-)}$, will be used to compensate for the displacement of Target 2. 
From Fig. \ref{target_motion}, we observe that the maximum error of the displacement compensator is $\varsigma^{(-)}+\varsigma$, where $\varsigma$ is less than $d_i^{2}$. Therefore, we can understand that the displacement compensator error decreases as $d_i^{2}$ decreases. When $\hat{d}^{12} \in \Omega_1$, our objective is mainly to encircle Target 1, where $d_i^{2}$ is relatively large. However, once $\hat{d}^{12}$ enters $\Omega_2$ and $\Omega_3$, i.e. the two drones start encircling Target 2, both $d_i^{2}$ and $\varsigma$ decrease. In this case, the displacement of Target 2 will roughly be equal to $\b{c}-\b{c}^{(-)}$.

Let the time set when $\hat{d}^{12}\in\Omega_2$ and $\hat{d}^{12}\in\Omega_3$, be denoted as $\Phi_{(\Omega_2,\Omega_3)}$. Furthermore, denote $k_m-\tilde{k} \in \Phi_{(\Omega_2,\Omega_3)}$, where $\tilde{k}$ represents the time it takes for the two drones to transition from encircling Target 1 to encircling Target 2.
 

 Therefore, based on the center $\b{c}$ and the variable $\varpi_{i, p,2}$ in \eqref{eq1-4} and \eqref{eq1-3}, the update law for the TPE of Target 2 can be designed as
\begin{subequations}
\begin{align}
&\hat{\bq}_i^2
=
\hat{\bq}_i^2(k|k-1)
+
K(\varpi_{i, p,2}
-\hat{\varpi}_{p,2}),\label{eq1-5-1}\\
&\hat{\bq}_i^2(k|k-1)
= \hat{\bq}_i^{2(-)}
+t(R_i^{(-)}\bu_{i}^{(-)}
+ \hat{\b{v}}_2^{(-)}),\label{eq1-5-2}\\
&\hat{\b{v}}_2^{(-)}=\frac{\check{v}_2(\b{c}-\b{c}^{(-)})}{\max\{\check{v}_2,\norm{\b{c}-\b{c}^{(-)}}\}}\delta\{k=k_m\},\label{eq1-5-3}\\
&\hat{\varpi}_{p,2}=\b{q}_{12}^\top\hat{\bq}_i^2(k|k-1),\label{eq1-5-4}
\end{align}
\end{subequations}
where $\hat{\bq}_i^2=[\hat{x}_2,\ \hat{y}_2,\ \hat{z}_2]^\top \in \R^3$ and $\hat{\bq}_i^2(k|k-1)\in \R^3$ are the estimated and predicted position of Target $j$ at instant $k$, respectively. $\hat{\varpi}_{p,2}$ represents the estimated position output of Target $j$.  $K\in \R^3$ is an adaptive estimator gain based on the recursive least-squares method. Specifically. 
\begin{equation}\label{eq1-6}
\begin{split}
K&=\bm{\eta}_1^{(-)}\b{q}_{12}(\gamma_1+\b{q}_{12}^\top\bm{\eta}_1^{(-)}\b{q}_{12})^{-1},
\end{split}
\end{equation}
where the covariance matrix $\bm{\eta}_1\in \R^{3\times 3}$ $(\bm{\eta}_1(0)>0)$ is defined as
\begin{equation}\label{eq1-7}
\begin{split}
\bm{\eta}_1^{-1}=&\gamma_1(\bm{\eta}_1^{(-)})^{-1}+\b{q}_{12}\b{q}_{12}^\top
\end{split}
\end{equation}
with a forgetting factor $\gamma_1$.

In addition, the term $\hat{\b{v}}_2$ in \eqref{eq1-5-3} is a displacement compensator. $\delta$ is Dirac delta function, i.e, $\delta\{k=k_m\}=1$ for $k=k_m$, otherwise $\delta\{k=k_m\}=0$.




 \subsection{Anti-drone control} \label{sec: asads}
First, considering real-time UAV attitude variations caused by dynamic environmental noise, the actual rotation matrix $\hat{R}_i(k)$ can be obtained based on real-time measurement angles $\phi,\theta, \psi$ of drone $i$. An attitude controller can then be designed using the following law to ensure that $R_i=I$. 
\begin{equation}
\begin{split}
\bar{\bu}_i=-\frac{1}{t}\Log(\hat{R}_i).
\end{split}
\end{equation}

Secondly, from the TPE, we can obtain an estimate of the iner-target distance $\hat{d}^{12}=\|\hat{\bq}_i^2-\b{q}_{i}^1\|$. To achieve the whole encirclement object, the AS-based ADCs can be designed in the following three cases.

\textbf{Case 1:} When $\hat{d}^{12}\in\Omega_1$, 
\begin{subequations}
\begin{align}
\bu_{1}=&\frac{R_1^{-1}}{t}\{(\alpha-1)\b{q}_{1}^{1}+\alpha\bm{\zeta}(r_1,\nu,k)\nonumber\\
&-\bm{\zeta}(r_1,\nu,+)+\b{v}_1^{(-)}\},\label{eq1-12-1}\\
\bu_{2}=&\frac{R_2^{-1}}{t}\{(\alpha-1)\b{q}_{2}^{1}-\alpha\bm{\zeta}(r_1,\nu,k)\nonumber\\
&+\bm{\zeta}(r_1,\nu,+)+\b{v}_1^{(-)}\},\label{eq1-12-2}
\end{align}
\end{subequations}
where $\alpha$ is the controller gain. $r_1$ is the given constant. 

\textbf{Case 2:} When $\hat{d}^{12}\in\Omega_2$, 
\begin{subequations}
\begin{align}
\bu_{1}=&\frac{R_1^{-1}}{t}\{(\alpha-1)\hat{\b{q}}_{1}^2+\alpha\bm{\zeta}(r_2,\nu,k)\nonumber\\
&-\bm{\zeta}(r_2,\nu,+)+\hat{\b{v}}_2^{(-)}\},\label{eq1-12-3}\\
\bu_{2}=&\frac{R_2^{-1}}{t}\{(\alpha-1)\hat{\b{q}}_{2}^2-\alpha\bm{\zeta}(r_2,\nu,k)\nonumber\\
&+\bm{\zeta}(r_2,\nu,+)+\hat{\b{v}}_2^{(-)}\},\label{eq1-12-4}
\end{align}
\end{subequations}
where $r_2$ is the given constant. 

\textbf{Case 3:} When $\hat{d}^{12}\in\Omega_3$, 
\begin{subequations}
\begin{align}
\bu_{1}=&\frac{R_1^{-1}}{t}\{(\alpha-1)\hat{\b{q}}_{1}^2+\alpha\bm{\zeta}(r_3,\nu,k)\nonumber\\
&-\bm{\zeta}(r_3^{(+)},\nu,+)+\hat{\b{v}}_2^{(-)}\},\label{eq1-12-5}\\
\bu_{2}=&\frac{R_2^{-1}}{t}\{(\alpha-1)\hat{\b{q}}_{2}^2-\alpha\bm{\zeta}(r_3,\nu,k)\nonumber\\
&+\bm{\zeta}(r_3^{(+)},\nu,+)+\hat{\b{v}}_2^{(-)}\},\label{eq1-12-6}
\end{align}
\end{subequations}
where $r_3=r_3^{(-)}-\frac{(r_2-\bar{r})(\hat{\b{v}}_2(\Omega_3(0))+\iota_2)}{z_2-z_1}$ with the estimated velocity $\hat{\b{v}}_2(\Omega_3(0))$ and time $\Omega_3(0)$ of Target 2 as it enters zone $\Omega_3$. $\iota_2$ is a given constant, and $r_3(\Omega_3(0))=r_2$.

\begin{remark} 
From a practical perception perspective, compared to other formation methods, having two guardian drones observe the target from opposite sides ensures maximum coverage and facilitates the acquisition of more information, see reference \cite{meng2016optimal}. From Fig. \ref{symmetrical_encircle}, it is easy to see that when two guardian drones encircle the target as the center, the position of the target is geometrically unique based solely on the two distance measurements. Therefore, the use of AS control can enhance positioning accuracy, which has been shown in reference \cite{meng2016optimal}.
\end{remark}

 \begin{figure}
\centering
  \includegraphics[width=7cm]{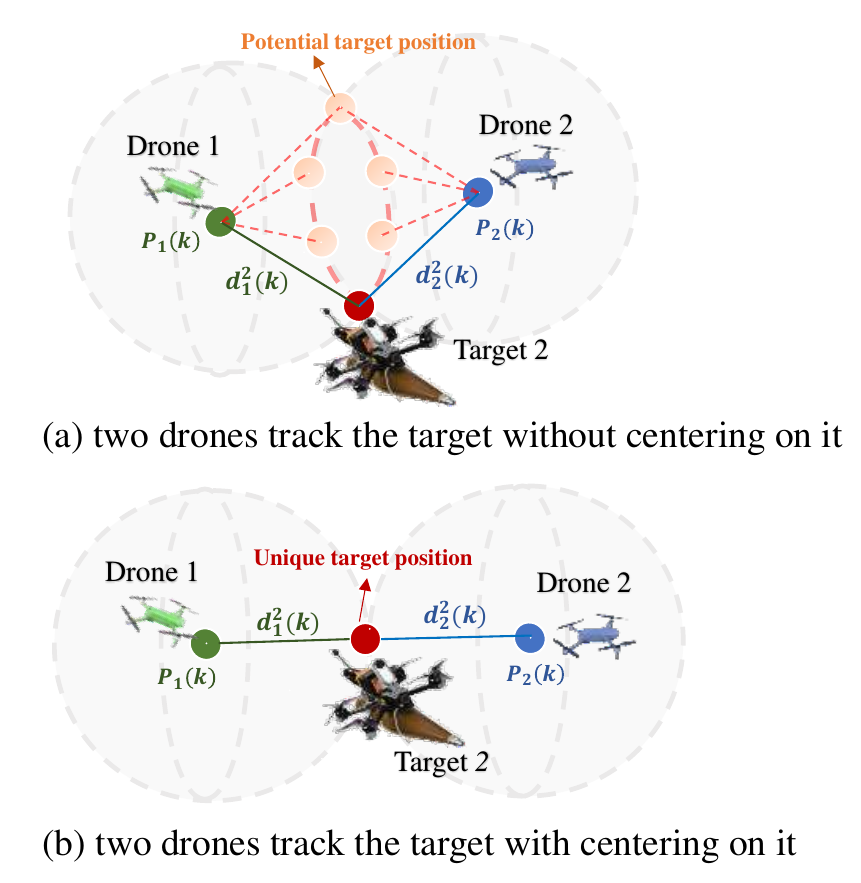}
 \caption{Two drones are symmetrical to the target.}
  \label{symmetrical_encircle}
  \vspace{-10pt}
\end{figure}

\subsection{Convergence analysis}
Considering the movement of the two drones under the AS-based ADC in the previous section, $\b{q}_{12}=\alpha\b{q}_{12}^{(-)}+2\alpha\bm{\zeta}(r,\nu,k)-2\bm{\zeta}(r^{(+)},\nu^{(+)},+)$ can be derived. Since $\alpha$ is given value and $\bm{\zeta}(r,\nu,k)$ satisfies PE, the following Lemma \ref{Persistently_exciting} can be obtained based on the Theorem IV.1 in \cite{nguyen2019persistently} and \cite{liu2024distance}.
 
\begin{lemma}(Persistently exciting)\label{Persistently_exciting}
The sequence $\{\b{q}_{12}(\kappa)\}, \kappa\in[k,k+N-1], \forall \kappa\in \Z$ is persistently exciting if the following inequality holds,
\begin{equation*}
\begin{split}
0<\hat{a}_2I_{n\times n}\leq\sum_{\kappa=k}^{k+N-1}\b{q}_{12}(\kappa) \b{q}_{12}^\top(\kappa)\leq\check{a}_2I_{n\times n}<\infty,
\end{split}
\end{equation*}
where $N$ is the motion period of encirclement. $\hat{a}_2$ and $\check{a}_2$ are the position contants.
\end{lemma}


\begin{theorem}
The two drones can estimate the position of Target 2 within an error bound if the exponential forgetting factor $\gamma_1$ satisfies $0<\gamma_1\leq\frac{1}{2}$.
\end{theorem}

\begin{proof}
See Appendix A.
\end{proof}

\begin{theorem}
The two drones can achieve the encirclement of all targets within a minimum error bound as the controller gain $\alpha$ satisfies the following condition $-\frac{1}{\sqrt{3}}<\alpha\leq\frac{1}{\sqrt{3}}$.
\end{theorem}

\begin{proof}
See Appendix B.
\end{proof}

\section{SIMULATIONS AND EXPERIMENTS}
\subsection{Numerical simulation}
A simulation example validates that the designed TPE and AS-based ADC can both act well. Here, we consider two drones, one ground cooperative car (Target 1), and one hostile drone (Target 2). For system models, the sampling period $t$ is given as $0.1 s$, and the initial values are as follows:  
\begin{equation*}
\begin{split}
&\bq_{12}(0)=[0,\ -2,\ 0.05]^\top,\bq^{2}_1(0)=[-6,\ -4,\ -1.45]^\top,\\
&\bq_{2}^2(0)=[-6,\ -2,\ -1.5]^\top.
\end{split}
\end{equation*}

\begin{figure}
\centering
  \includegraphics[width=8cm]{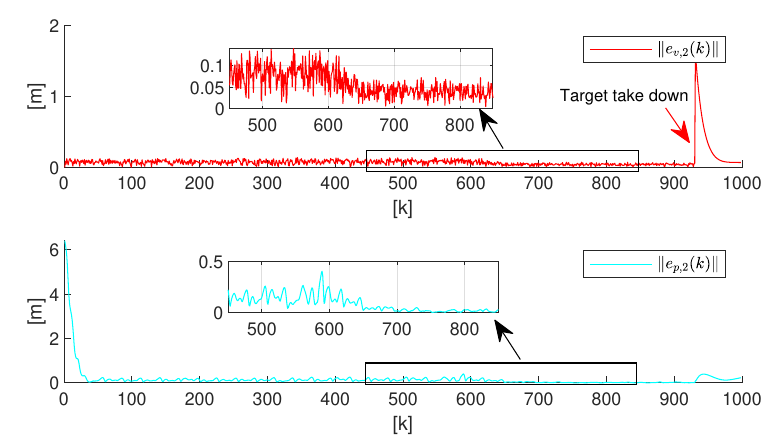}
 \caption{The trajectories of the velocity compensation error $\be_{v,2}(k)$ and the estimation error $\be_{p,2}(k)$ for Target 2.}
  \label{estimation_error}
\end{figure} 

\begin{figure}
\centering
  \includegraphics[width=8cm]{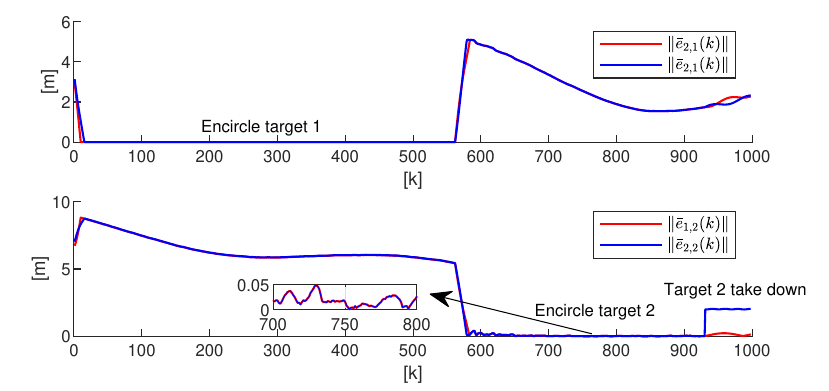}
 \caption{The trajectories of the AS-based encirclement tracking errors $\be_{1}(k)$ and $\be_{2}(k)$ for Target 1 and Target 2.}
  \label{encirclement_error}
\end{figure} 


Based on Theorem 1 and Theorem 2, the following parameters are set, i.e. $\gamma_1=0.45$ and $\alpha=-0.001$.  The vertical motion function can be designed as $g(k)=0.3\cos(\frac{1}{8} k\pi)$ and $\bar{r}=0.1$, $r_1=5.8$, $r_2=3$ and $\iota=0.03$. Additionally, Drone 1 and Drone 2 will project Target 1 onto the excepted height plane, i.e. $h=0.5$.

The simulation results are shown in Fig. \ref{estimation_error} to Fig. \ref{encirclement_error}. The error trajectory in Fig. \ref{estimation_error} shows that the position estimation error for Target 2 is relatively large when encircling Target 1. Once Target 2 enters the warning area, the drones begin encircling it, and the position estimation error for Target 2 decreases to around 0.05 meters, which indicates that AS-based encirclement control contributes to improving the drones' positioning accuracy for targets. Additionally, from the tracking error trajectory in Fig. \ref{encirclement_error}, it can be seen that the two drones can quickly switch from encircling Target 1 to encircling Target 2. For cooperative Target 1, the encirclement tracking error is close to 0, while for non-cooperative Target 2, the encirclement tracking error is within 0.05 meters.

\begin{figure}
\centering
  \includegraphics[width=8cm]{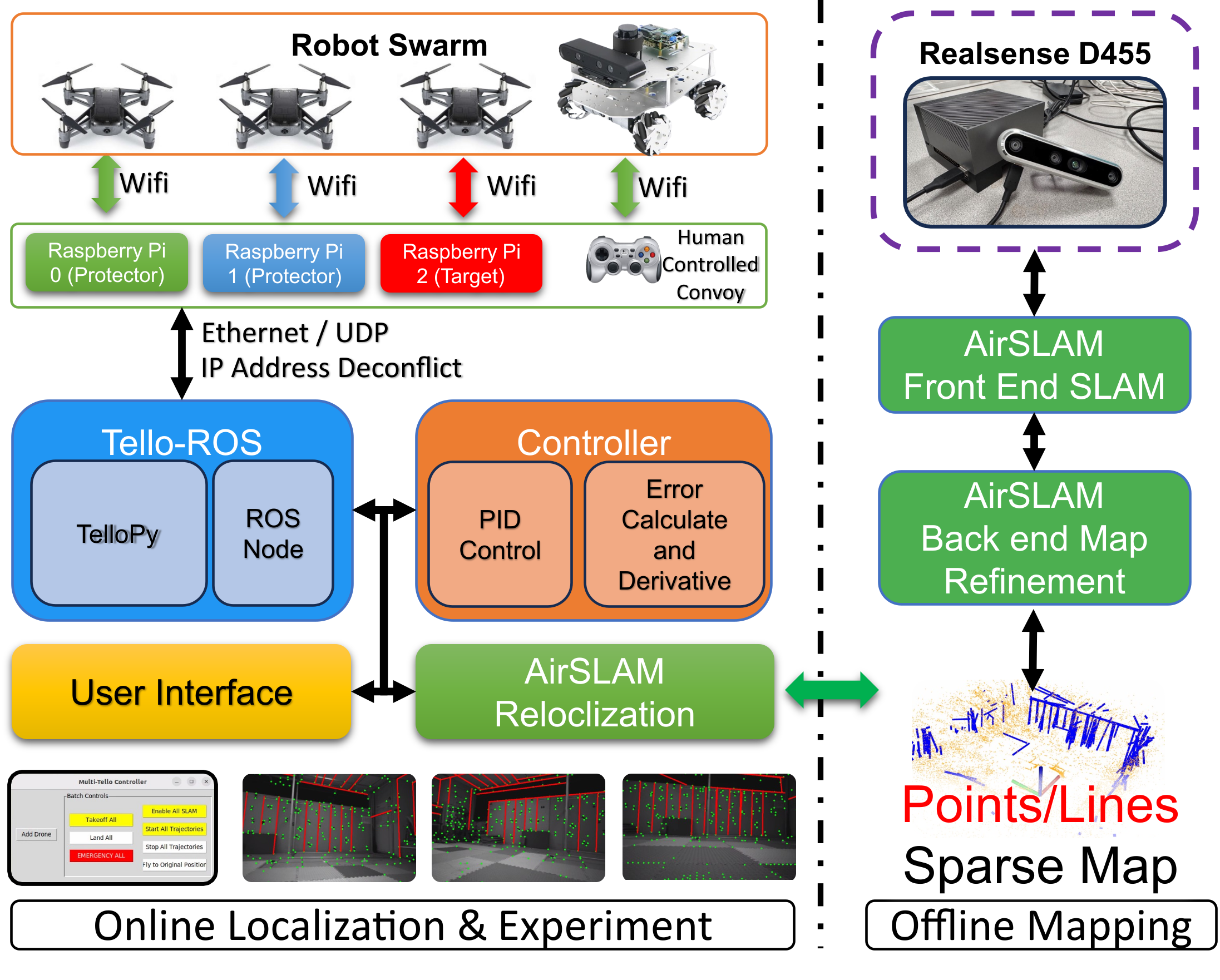}
 \caption{Experiment setup.}
  \label{experiment_setup}
\end{figure} 

\subsection{Real-world UAV-based experiment}
\textbf{Experiment setup:} To validate the proposed solution, a real-world demonstration is conducted, as shown in Fig. \ref{experiment_setup}. The setup of the experiment consists of a surface patrol vehicle that simulates the vulnerable ship, two guard drones, and a hostile drone. All drones used in this study are low-cost Tello drones, chosen to minimize potential damage in the event of a simulated collision. However, these drones lack onboard payload capacity for a microphone array and a PC.
To address this limitation, each Tello drone is wirelessly connected to a Raspberry Pi 4B, which manages the configuration and deconfliction of IP addresses. This setup enables simultaneous control and video streaming for all drone units. Once the IP addresses are assigned, an NUC 11 i7-1165G7 is utilized to process individual visual SLAM instances and issue control commands. Meanwhile, a separate NUC is responsible for managing the hostile drone's behavior, which is autonomously controlled.

All robots operate with multiple instances of AirSLAM nodes \cite{xu2025airslam, xu2024efficient}, processing input images resized to 480p at 30 Hz. The local point clouds generated by each drone are aligned with a pre-built map constructed using offline AirSLAM mapping solutions \cite{xu2025airslam}, as illustrated in Fig. \ref{experiment_setup}. Guardian drones estimate their relative positions by referencing common global point clouds, while pseudo range distances between the guardian UAV, hostile UAV, and surface patrol vehicle are inferred indirectly. Unlike motion capture systems, camera-based range measurements introduce inherent noise, closely resembling real-world measurement uncertainty. 

\begin{figure*}
\centering
  \includegraphics[width=1\linewidth]{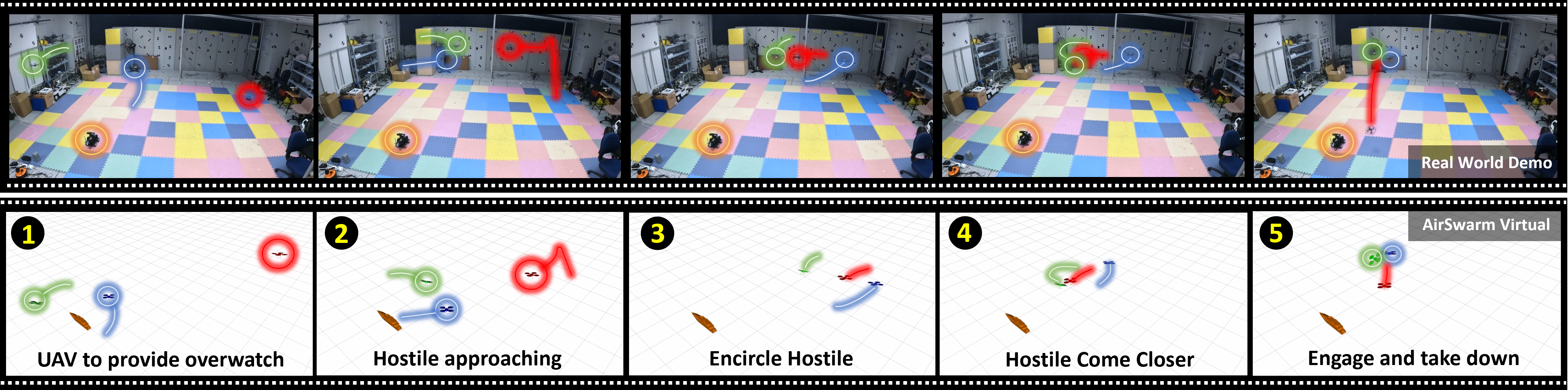}
  \vspace{-20pt}
  \caption{Real-world and virtual UAV swarm tactics for overwatch, approach, encirclement, engagement, and takedown.}
  \label{Real_world_demonstrations}
\end{figure*}

\textbf{Scenario setup:} The experiment simulates a defense scenario in which a surface vehicle, representing a vulnerable ship, traverses a designated area under human control. When a remote-controlled hostile drone attempts a kamikaze-style attack on the convoy, two guardian UAVs respond by encircling the hostile drone based on its relative distance. If the hostile drone approaches too closely, the guardian UAVs execute a direct kinetic collision to neutralize the threat. The overall concept of the real-world demonstration is illustrated in Fig. \ref{Real_world_demonstrations}.

The real-world demonstration lasted approximately 50 seconds, during which a buffering mechanism was implemented to capture the state of each drone at every observation. Uninitialized states were marked as unavailable and the buffered data was published for visualization. Theoretically, the system could collect up to 8,500 state feedback instances. However, due to bandwidth limitations and processing delays, approximately 6,000 samples were recorded over the 50-second duration, resulting in an average effective observation rate of 20 to 25 Hz per drone.

During the first 500 measurements, both guardian drones prioritized reaching their designated orbital positions around Target 1, maintaining minimal deviation while disregarding Target 2, as shown in Fig \ref{Experiment_error}. By the 2,500th measurement, the guardian swarm detected a threat from Target 2 and began maneuvering toward it. This phase was marked by a crossover event, where the two guardian drones repositioned themselves between Targets 1 and 2. From the 3,000th to the 5,000th measurement, the guardian drones maintained an orbit around Target 2. As the hostile entity approached dangerously close, the encirclement radius was gradually reduced, ultimately leading to a controlled collision to neutralize the threat.

More simulation and experimental details can be found in the URL \url{https://youtu.be/5eHW56lPVto}.

\begin{figure}
\centering
  \includegraphics[width=8cm]{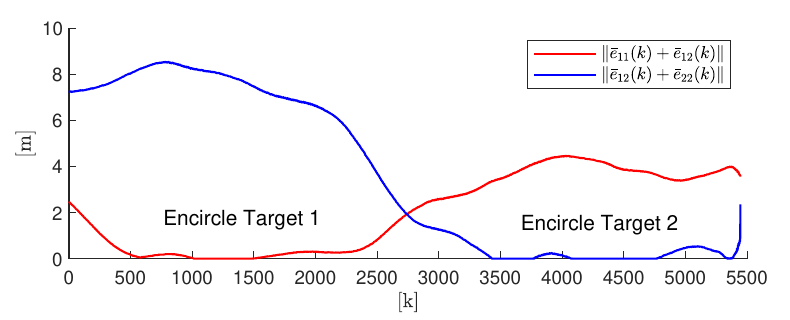}
 \caption{The trajectories of the AS-based encirclement tracking errors $\be_{1}(k)$ and $\be_{2}(k)$ in experiment.}
  \label{Experiment_error}
\end{figure}

 \section{Conclusions}
In this research, we presented an innovative aerial drone-to-drone encirclement and interception strategy designed as a robust vehicle protection solution without any surface vehicle guidance. This approach leveraged AS-based control strategies and simple range-based measurements, primarily derived from hostile UAV noise information. Our results demonstrated the effectiveness of a specially designed perception-aware controller, which enabled drones to autonomously neutralize hostile threats without requiring external surface vehicle guidance. This autonomous capability proved particularly valuable in safeguarding vulnerable ships, enhancing situational awareness, and improving countermeasures against aerial threats.

\section*{Appendix A}
\section*{Proof of Theorem 1}
Denote the relative position estimation error of Target $2$ as $\be_{p,2}\triangleq{\bq}_i^2-\hat{\bq}_i^2$. Then, recalling the relative position model in \eqref{eq1-5} and the position estimator in \eqref{eq1-5-1}, the dynamics of $\be_{p,2}$ can be further obtained as
\begin{equation}\label{eq1-14}
\begin{split}
\be_{p,2}^{(+)}=&A_1(\be_{p,2}+t(\b{v}_2-\hat{\b{v}}_2)),
\end{split}
\end{equation}
where $A_1=I_n-K^{(+)}(\b{q}_{12}^{(+)})^\top$.


The Lyapunov function (LF) candidate can be chosen as
$V_{11}=\be_{p,2}^\top\bm{\eta}_1^{-1}\be_{p,2}$. Furthermore, denoted the velocity compensation error of Target $2$ as $\be_{v,2}\triangleq\b{v}_2-\hat{\b{v}}_2$ and based on the Cauchy-Schwarz inequality, the difference of $V_{11}$ can be deducted as
\begin{equation*}
\begin{split}
\triangle V_{11}=&V_{11}^{(+)}-V_{11}\\
\leq&2\be_{p,2}^\top A_1^\top(\bm{\eta}_1^{(+)})^{-1}A_1\be_{p,2}+2t^2\be_{v,2}^\top A_1^\top(\bm{\eta}_1^{(+)})^{-1}\\
&\times A_1\be_{v,2}-\be_{p,2}^\top\bm{\eta}_1^{-1}\be_{p,2}.
\end{split}
\end{equation*}


According to the matrix inversion lemma, we can obtain $A_1=\gamma_1\bm{\eta}_1^{(+)}\bm{\eta}_1^{-1}$. The differences of $V_{11}$ can be re-obtained as
\begin{equation*}
\begin{split}
\triangle V_{11}=&V_{11}^{(+)}-V_{11}\\
\leq&(2\gamma_1-1)\be_{p,2}^\top\bm{\eta}_1^{-1}\be_{p,2}+2\gamma_1t^2\be_{v,2}^\top\bm{\eta}_1^{-1}\be_{v,2}.
\end{split}
\end{equation*}

Referring back to Lemma \ref{Persistently_exciting} and formation \eqref{eq1-7}, through a straightforward recursive process, we establish the upper and lower bounds of the error covariance matrix $\bm{\eta}_1^{-1}$, i.e. $0 < \hat{\beta}_1I_n \leq \bm{\eta}_1^{-1} \leq \check{\beta}_1I_n$ for all $k\geq N-1$. Detailed derivations can be found in the proof of Lemma 1 in \cite{johnstone1982exponential}. 

Assume that $\norm{\be_{v,2}}\leq \varrho_2$, the convergence of position estimation error $\be_{p,2}$ can be further analyzed. As $k\rightarrow\infty$ and $0 < 2\gamma_1 \leq 1$, we arrive at the result $\norm{\be_{p,2}}^2 \leq \varrho_1$, where $\varrho_1=\frac{2t^2\gamma_1 \check{\beta}_1\varrho_2^2}{(1-2\gamma_1)\hat{\beta}_1}$,  which implies that the target position estimation error is exponentially bounded.

\section*{Appendix B}
\section*{Proof of Theorem 2}
Based on Problem 1, we can define the AS-based tracking encirclement error and the encirclement errors for Drone 1 and Drone 2 as $\bar{\be}_{1,j}\triangleq\b{q}_{1}^{j}+\bm{\zeta}(r,\nu,k)$ and $\bar{\be}_{2,j}\triangleq\b{q}_{2}^{j}-\bm{\zeta}(r,\nu,k)$, respectively. Then, for target 2, the following dynamics of $\bar{\be}_{1,2}$ and $\bar{\be}_{2,2}$ can be further derived based on the relative position model of drones and Target 2 in \eqref{eq1-5}.
\begin{equation*}
\begin{split}
\bar{\be}_{1,2}^{(+)}=&\alpha\bar{\be}_{1,2}+(\alpha-1)\be_{p,2}-(\be_{v,2}^{(-)}+\Delta\b{v}_2),\\
\bar{\be}_{2,2}^{(+)}=&\alpha\bar{\be}_{2,2}+(\alpha-1)\be_{p,2}-(\be_{v,2}^{(-)}+\Delta\b{v}_2).
\end{split}
\end{equation*}

The LFs for the encirclement errors can be chosen as $V_{21}=\norm{\bar{\be}_{1,2}}^2$ and $V_{22}=\norm{\bar{\be}_{2,2}}^2$. Then, the differences of $V_{21}$ and $V_{22}$ can be
obtained as
\begin{equation*}
\begin{split}
\triangle V_{21}
\leq& (3\alpha^2-1)V_{21}+3(\alpha-1)^2\norm{\be_{p,2}}^2\\
&+3\norm{\be_{v,2}+\Delta\b{v}_2}^2,\\
\triangle V_{22}
\leq& (3\alpha^2-1)V_{22}+3(\alpha-1)^2\norm{\be_{p,2}}^2\\
&+3\norm{\be_{v,2}+\Delta\b{v}_2}^2.
\end{split}
\end{equation*}

Based on the condition $-\frac{1}{\sqrt{3}}<\alpha\leq\frac{1}{\sqrt{3}}$, we also can make sure that $\norm{\bar{\be}_{1,2}}^2=\norm{\bar{\be}_{2,2}}^2 \leq \varepsilon_{2,2}$ for $k\rightarrow\infty$, where $\varepsilon_{2,2}=\frac{3(\alpha-1)^2\varrho_1+3(\varrho_2+\check{v}_2)^2}{1-3a^2}$.

Furthermore, the dynamics of the AS-based tracking encirclement error and the encirclement errors for Target 1 can be derived as 
\begin{equation*}
\begin{split}
\bar{\be}_{1,1}^{(+)}=&\alpha\bar{\be}_{1,1}+\Delta\b{v}_1,\\
\bar{\be}_{2,1}^{(+)}=&\alpha\bar{\be}_{2,1}+\Delta\b{v}_1.
\end{split}
\end{equation*}

Similarly, based on the condition $-\frac{1}{\sqrt{3}}<\alpha\leq\frac{1}{\sqrt{3}}$, we have $\|\bar{\be}_{1,1}\|^2 =\|\bar{\be}_{2,1}\|^2 \leq \varepsilon_{2,1}$ with $\varepsilon_{1,1} = \frac{8\check{v}_1^2}{1-2a^2}$ and $\varepsilon_{2,1} = \frac{2\check{v}_1^2}{1-2a^2}$ for $k\rightarrow\infty$. 





\balance
\bibliographystyle{IEEEtran}
\bibliography{IEEEfull}

\end{document}